\documentclass{article}

\usepackage{microtype}
\usepackage{graphicx}
\usepackage{subfigure}
\usepackage{booktabs} 
\usepackage{multirow}

\usepackage{hyperref}

\usepackage[accepted]{icml2023}

\usepackage{amsmath}
\usepackage{amssymb}
\usepackage{mathtools}
\usepackage{amsthm}
\usepackage{xcolor,colortbl}

\usepackage[capitalize,noabbrev]{cleveref}

\theoremstyle{plain}
\newtheorem{theorem}{Theorem}[section]
\newtheorem{proposition}[theorem]{Proposition}

\theoremstyle{definition}

\theoremstyle{remark}

\DeclarePairedDelimiterX{\infdivx}[2]{(}{)}{%
  #1\;\delimsize\|\;#2%
}
\newcommand{\infdiv}{D_{KL}\infdivx}

\usepackage[textsize=tiny]{todonotes}
\definecolor{mygray}{gray}{0.9}
\newcommand{\shade}{\cellcolor{mygray}}

\icmltitlerunning{On genuine invariance learning without weight-tying}

\begin{document}

\twocolumn[
\icmltitle{On genuine invariance learning without weight-tying}




\begin{icmlauthorlist}
\icmlauthor{Artem Moskalev}{delta}
\icmlauthor{Anna Sepliarskaia}{twente}
\icmlauthor{Erik J. Bekkers}{uva}
\icmlauthor{Arnold Smeulders}{uva}
\end{icmlauthorlist}

\icmlaffiliation{delta}{UvA-Bosch Delta Lab, University of Amsterdam, The Netherlands}
\icmlaffiliation{twente}{University of Twente \& Booking.com, The Netherlands}
\icmlaffiliation{uva}{University of Amsterdam, The Netherlands}

\icmlcorrespondingauthor{Artem Moskalev}{ammoskalevartem@gmail.com}

\icmlkeywords{Machine Learning, ICML}

\vskip 0.3in
]



\printAffiliationsAndNotice{} 

\begin{abstract}
In this paper, we investigate properties and limitations of invariance learned by neural networks from the data compared to the genuine invariance achieved through invariant weight-tying. To do so, we adopt a group theoretical perspective and analyze invariance learning in neural networks without weight-tying constraints. We demonstrate that even when a network learns to correctly classify samples on a group orbit, the underlying decision-making in such a model does not attain genuine invariance. Instead, learned invariance is strongly conditioned on the input data, rendering it unreliable if the input distribution shifts. We next demonstrate how to guide invariance learning toward genuine invariance by regularizing the invariance of a model at the training. To this end, we propose several metrics to quantify learned invariance: \textit{(i)} predictive distribution invariance, \textit{(ii)} logit invariance, and \textit{(iii)} saliency invariance similarity. We show that the invariance learned with the invariance error regularization closely reassembles the genuine invariance of weight-tying models and reliably holds even under a severe input distribution shift. Closer analysis of the learned invariance also reveals the spectral decay phenomenon, when a network chooses to achieve the invariance to a specific transformation group by reducing the sensitivity to \textit{any} input perturbation. GitHub: \url{https://github.com/amoskalev/ginvariance}
\end{abstract}

\section{Introduction}

\begin{figure}[t!]
  \centering
    \includegraphics[width=0.95\linewidth]{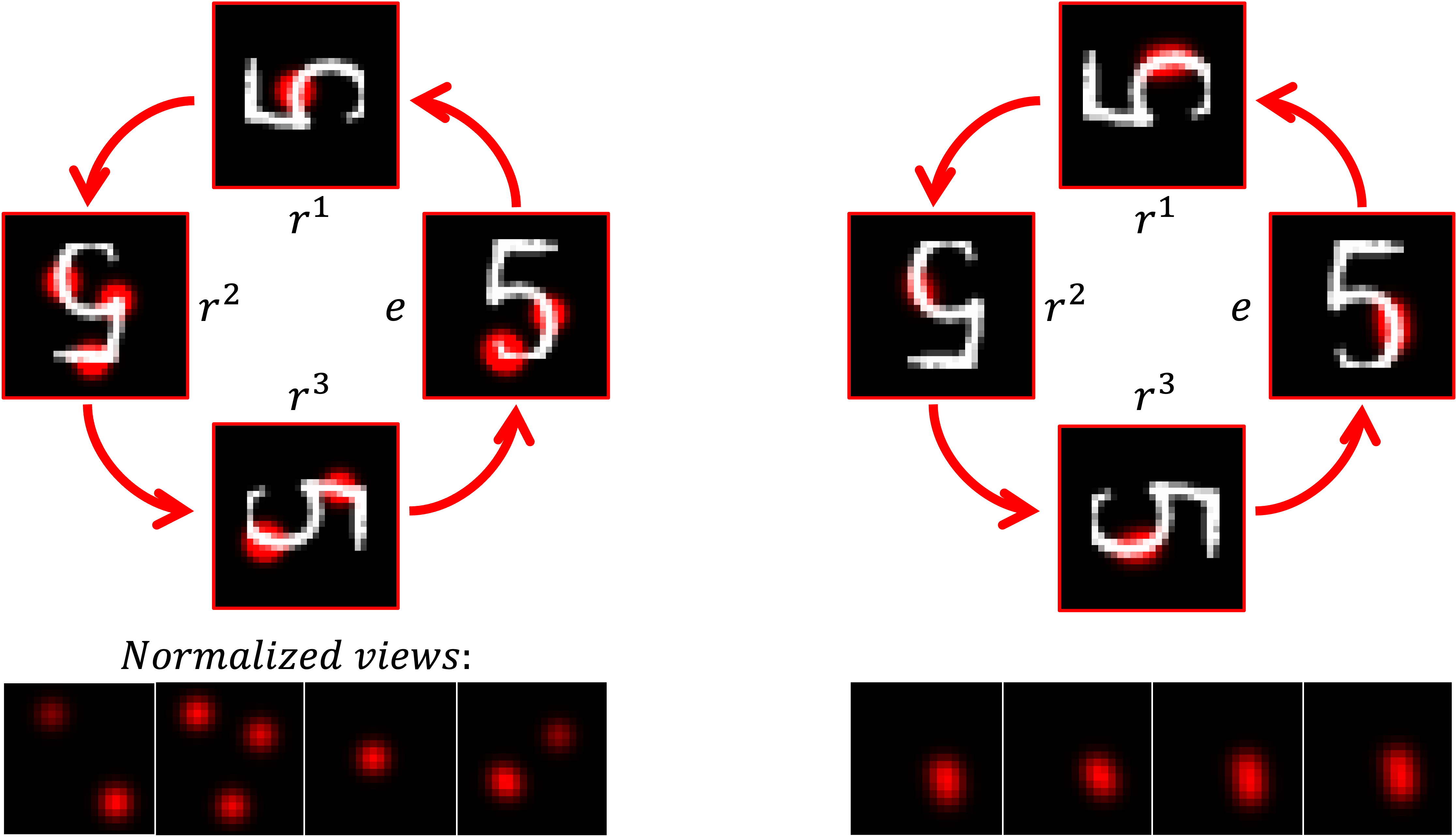}
    \caption{\textbf{Left:} The network learns a separate set of features for each of the orientations as indicated by divergent saliency maps. \textbf{Right:} saliency maps of the networks with group-invariant weight-tying. \textit{In both cases, predicted class distributions are identical for all of the orientations.} \textbf{Bottom row:} saliency maps normalized to a common orientation.}
  \label{fig:fig1}
\end{figure}

The ability to abstract from irrelevant details and focus on core aspects is a foundational property of intelligent systems. Invariance, a crucial step of this abstraction process, enables neural networks to recognize patterns regardless of their transformations. Achieving effective invariance is vital for the robust performance of deep learning models.

There exist two approaches for invariance in neural networks: invariant weight-tying and learning invariance from data. Networks with built-in invariant weight-tying \cite{cohen16gcnn,worrall2017harmonic,e2cnn,worrall2019deep,sosnovik2020sesn,bekkers2020bspline} offer \textit{genuine} invariance, but require knowledge of geometrical priors and incurs high computational and memory costs \cite{sosnovik2021disco,sosnovik2021transform}. Alternatively, neural networks can learn invariance directly from data. Recent works \cite{olah2020naturally,benton2020learning,moskalev2022liegg} demonstrate that neural networks successfully learn invariant priors without any architectural modifications. However, the nature of learned invariance remains largely unexplored, particularly regarding whether it resembles the genuine invariance of weight-tying methods at any level. Consequently, this raises concerns about how much we can rely on the learned invariance when operating conditions evolve. In this work, we investigate the properties of learned invariance to better understand its potential and limitations.

To investigate properties of the learned invariance, we adopt the group theoretical perspective and analyze invariance learning without weight-tying constraints. Firstly, we analyze the saliency maps of no weight-tying networks with learned invariance. We demonstrate that even when such networks learn to correctly classify samples on a group orbit, the underlying decision-making process does not attain genuine invariance, see Figure \ref{fig:fig1}. Instead of learning genuinely invariant weight-tying, unconstrained networks choose to learn a separate set of features for each of the transformations from a group orbit, even when invariance is enforced by strong data augmentation. This results in learned invariance being strongly conditioned on the input data. Consequently, the effectiveness of learned invariance degrades rapidly when operating conditions evolve, e.g. under input distribution shift. This renders neural networks with learned invariance less reliable.

Secondly, we tackle the problem of aligning learned invariance with the genuine invariance of weight-tying networks. To do so, we propose several measures to quantify the invariance error; we next use those measures to regularize the task loss to promote genuine invariance learning. We conduct experiments with rotation and translation groups, and we show that the proposed regularization significantly aligns learned invariance with the genuine invariance achieved through the weight-tying. However, the alignment also induces performance decay on a downstream task. This presents a new challenging problem of achieving genuine invariance by learning through data augmentation and specialized losses, while also maintaining the downstream task performance.

Thirdly, we investigate the performance decay under the learned invariance. To this end, we analyze the training dynamics of the invariance error minimization from the perspective of the gradient flow. We show that minimizing the invariance error without weight-tying implicitly promotes attaining the invariance to a certain group of transformations by reducing the sensitivity to \textit{any} input perturbation. This has an effect similar to training a network with a large weight decay, which motivates the performance drop. We conduct experiments and demonstrate that this phenomenon holds for various transformations and various forms of invariance error minimization.

To sum up, we make the following contributions:
\begin{itemize}
    \item We demonstrate that data-driven invariance learning fails to learn genuine invariance as in weight-tying networks.
    \item We show that it is possible to attain genuine invariance through invariance regularization, but at the cost of the downstream task performance.
    \item We attribute the performance decay under learned invariance to the training dynamics of the invariance error minimization, which constrains the sensitivity of a network to input perturbations in general.
\end{itemize}
\section{Related work}

\paragraph{Weight-tying invariance} Weight-tying is the approach for invariance in neural networks that is based on the concept of group equivariant networks \cite{cohen16gcnn}. Group equivariant networks explicitly embed equivariance, or invariance as a special case, for specific transformation groups into a network architecture. The principle traces back to convolutional networks \cite{LeCun1999cnn} which incorporate translation symmetry. The scope of equivariant networks has since expanded to include other transformations such as rotations \cite{cohen16gcnn,worrall2017harmonic,e2cnn,jenner2022steerable}, permutations \cite{Zaheer2017deepsets}, and scaling \cite{worrall2019deep,sosnovik2020sesn,bekkers2020bspline,Sosnovik_2021_WACV,sosnovik2021disco,sosnovik2021transform}. Another line of work focuses on advancing group equivariant networks by enabling them to learn symmetries directly from the data \cite{anselmi2019symmetry,zhou2020meta,dehmamy2021automatic,sanborn2023bispectral}. This allows the model to adjust to specific symmetries present in the training dataset, eliminating the need for prior knowledge of geometrical priors. Yet, these methods still require modifying the architecture to train invariance.

In this work, we treat the weight-tying methods as oracle invariance learners and investigate whether networks without specific architectural modifications can learn the degree and quality of invariance comparable to the weight-tying approaches.

\paragraph{Data-driven invariance learning} Another approach for achieving invariance is to learn it directly from the data. Recent and earlier works \cite{Goodfellow2009Measuring,Lenc2014understanding,benton2020learning,moskalev2022liegg,kvinge2022in} demonstrate that neural networks can learn invariance without relying on specialized architectural modifications. Additionally, training with data augmentation has long been seen as a method to increase invariance of a model for input transformations \cite{perez2017effectiveness,shorten2019survey,cubuk2018autoaugment}. Invariance learning that does not require specialized architectural modification is advantageous as it does not incur additional memory or computational costs. However, the nature of the learned invariance and its comparability to the genuine invariance obtained through the weight-tying remains an open question. The properties and reliability of such learned invariance are not well understood, which motivates the study in this paper.
\section{Learning invariances from data}

We take a group-symmetry perspective on data-driven invariance learning when the downstream task is classification. That is to say, we define a set of transformations to be a symmetry group a network needs to learn to be invariant to when classifying input signals. We start by briefly introducing group symmetry and invariance.

\subsection{Group symmetry}

\paragraph{Group} A group $\langle \mathcal{G}, \circ \rangle$ is a set $\mathcal{G}$ with a group binary operation $\circ$ called the group product. For convenience, it is common to simplify the notation $a \circ b$ to $ab$. The group product combines two elements from $\mathcal{G}$ to a new element so that the following group axioms are satisfied. \textit{Closure:} for all $a,b \in \mathcal{G}$, the element $ab \in \mathcal{G}$. \textit{Associativity:} for all $a,b,c \in \mathcal{G}$, $(ab)c = a(bc)$. \textit{Identity:} there is an element $e \in \mathcal{G}$ such that $ea = ae = a$ for every element $a \in \mathcal{G}$. \textit{Inverse:} for each $a \in \mathcal{G}$ there exist $a^{-1} \in \mathcal{G}$ such that $a^{-1}a = aa^{-1} = e$. 

\paragraph{Group actions \& Symmetry} Group actions are a way of describing symmetries of objects using groups. A group action of a group $\mathcal{G}$ on a set $\mathcal{X}$ maps each element $g \in \mathcal{G}$ and each element $x \in \mathcal{X}$ to an element of $\mathcal{X}$ in a way that is compatible with the group structure. In other words, $ex = x$ and $(g_1 g_2) x = g_1 (g_2 x)$ for any $x \in \mathcal{X}$ and $g_1, g_2 \in \mathcal{G}$.

\paragraph{Group-invariance} Group-invariance is a property of a function $f: \mathcal{X} \rightarrow \mathcal{Y}$ under a group action from a group $\mathcal{G}$. A function $f$ is said to be group-invariant if $f(gx) = f(x)$ for $g \in \mathcal{G}$. This means that the value of $f$ at $x$ is unchanged by the action of any group element.

\paragraph{Group orbit} The group orbit of an element $x \in \mathcal{X}$ under a group action from a group $\mathcal{G}$ is the set of all points in $\mathcal{X}$ that can be reached by applying the group action on $x$. More formally, the orbit of $x$ is defined as the set $\mathcal{O}_x = \{gx | g \in \mathcal{G}\}$. This concept encapsulates the idea that the group action can move the element $x$ around within the set, and the orbit describes all the possible positions $x$ can be moved to by the group action.

\subsection{Measuring learned invariance}

Next, we explain how to measure group-invariance learned by a neural networks from the data. We assume we are given a neural network $f: \mathcal{X} \rightarrow \mathcal{Y}$ that maps inputs to logits, a group $\mathcal{G}$ and the dataset $\mathcal{D}$. We define three types of measures: \textit{(i) predictive distribution invariance} to measure the average change of a network's output distribution when a symmetry transformation is applied, \textit{(ii) logit invariance} to measure the change of raw network's logits and \textit{(iii) saliency invariance similarity} to evaluate the consistency of network's decisions under group transformations.

\paragraph{Predictive distribution invariance} Since the downstream task of interest is classification, it is natural to measure the invariance by evaluating the shift of the predictive distribution when transformations from a group orbit are applied. Practically, we can utilize Kullback–Leibler divergence between output softmax-distributions of $f(x)$ and $f(gx)$. With this, we can write the predictive distribution invariance error $DI_f$:

\begin{equation}
\label{eq:distr_inv}
    DI_f(\mathcal{D}, \mathcal{G}) = \sum_{x \sim \mathcal{D}} \sum_{g \sim \mathcal{G}} \infdiv{u_{x}}{q_{gx}}
\end{equation}

where $u_x$ and $q_{gx}$ denote the \textit{softmax} applied to the logits $f(x)$ and $f(gx)$ respectively.

Since $DI_f$ operates directly on the level of predictive distributions, it is the most useful to evaluate the invariance tackled to the downstream classification task. 

\paragraph{Logit invariance} Next, we define the logit invariance error to measure the shift of raw logits under group actions. Practically, we utilize average squared $L_2$ distance between the logits $f(x)$ and $f(gx)$:

\begin{equation}
\label{eq:logit_inv}
    LI_f(\mathcal{D}, \mathcal{G}) = \sum_{x \sim \mathcal{D}} \sum_{g \sim \mathcal{G}} \frac{1}{2} \big{\|} f(x) - f(gx) \big{\|}_{2}^{2}
\end{equation}

Note that the logit invariance error is a more strict invariance measure compared to $DI_f(\mathcal{D}, \mathcal{G})$. This is due to a scalar addition invariance of the predictive softmax-distribution. That means $LI_f(\mathcal{D}, \mathcal{G}) = 0$ implies $DI_f(\mathcal{D}, \mathcal{G}) = 0$, but not vice versa. With this, the logit invariance error is the most useful to characterize the absolute invariance of a function to group transformations regardless of a particular downstream task.

\paragraph{Saliency invariance similarity} Lastly, we propose saliency invariance similarity $SI_f$ to measure the consistency of the decision-making process of a neural network under input transformations. Let $m_f: \mathcal{X} \times \mathcal{Y} \rightarrow \mathcal{S}$ be a saliency map function for the network $f$ \cite{sundararajan2017axiomatic,mundhenk2019efficient}. We then compute the similarity between $m_{f}(x)$ and $g^{-1} m_{f}(gx)$, where $g^{-1}$ is needed to ensure a common orientation of the saliency maps. Practically, we adopt the cosine similarity and compute the average saliency similarity as:

\begin{equation}
\label{eq:saliency_inv}
SI_f(\mathcal{D}, \mathcal{G}) = \sum_{x \sim \mathcal{D}} \sum_{g \sim \mathcal{G}} \frac{m_{f}(x) \cdot g^{-1}m_{f}(gx)}{\| m_{f}(x) \|_2 \| g^{-1}m_{f}(gx) \|_2}
\end{equation}

The saliency invariance similarity $SI_f(\mathcal{D}, \mathcal{G})$ reflects how much the direction of the most important features, that a network bases its decisions on, change under transformations from a group orbit. Saliency invariance similarity differs from the previous two metrics as it considers the structure of the input data and not just the output of the network. This makes $SI_f$ particularly useful to understand how group transformations alter a network's internal decision-making process.

\subsection{Invariance regularization}

\paragraph{Constrained invariance learning} We next consider the task of facilitating learning invariances from the data. The natural way to do so is to optimize the task performance subject to a low invariance error. Practically, with the dataset $\mathcal{D}$ and a group of interest $\mathcal{G}$, invariance learning boils down to the constrained optimization approach $\text{min}_{\theta} \textit{ } \mathcal{L}_f(\mathcal{D}) \textit{ s.t. } {I}_f(\mathcal{D}, \mathcal{G})=0$; then, to train a neural network, we can simply optimize the relaxation:

\begin{equation}
    \text{min}_{\theta} \textit{ } \mathcal{L}_f(\mathcal{D}) + \nu {I}_f(\mathcal{D}, \mathcal{G})
\end{equation}

where $\theta$ denotes the parameters of $f$, $\mathcal{L}_f$ is a downstream task loss functions, ${I}_f(\mathcal{D}, \mathcal{G})$ is an invariance regularizer with respect to the group $\mathcal{G}$ and $\nu$ regulates how much invariance we want to achieve at the training. Practically, we observed that using the logit invariance error as a regularizer provides better overall invariance and accuracy than other forms of the invariance error, see Section \ref{sec:sdecay}.

Adding an invariance-regularizer to the original loss yields a simple approach to facilitate data-driven invariance learning. We experimentally demonstrate that invariance regularization significantly improves the quality of learned invariance, closing the gap with the genuine invariance of weight-tying methods. However, we also observe that the improvement in the quality of invariance comes at the cost of downstream task performance, as we demonstrate in Section \ref{sec:sdecay}.

\paragraph{Invariance-induced spectral decay} In order to analyze the causes of performance decay under the invariance error minimization, we analyze the training dynamics of the learned invariance through the lens of its gradient flow. We show that a neural network opts for achieving the invariance to a particular transform group by reducing the sensitivity to any input variations. We use a maximum singular value $\sigma_{\text{max}}$ of network's weights as a sensitivity measure \cite{yoshida2017spectral,Khrulkov_2018_CVPR}; and we analyze the gradient flow for the logit invariance error with a class of linear neural networks. We firstly show that the logit invariance error minimization implicitly constrains the maximum singular value of network's weights, thereby reducing its input sensitivity. Then, we experimentally demonstrate that this result also holds for more complex neural networks and the various forms of invariance errors.

\vspace{7mm}

Consider a linear neural network $h(x) = Wx$. Without loss of generality, we analyze the sensitivity to the action of a single group element $g$, instead of the full orbit of the group $\mathcal{G}$. Let $G$ be a linear representation of the group acting on $x$ and consider invariance error minimization over $t$ steps.

\begin{proposition}[\textit{Invariance-induced spectral decay}]
\label{Ipq_theorem}
Logit invariance error minimization implies $\sigma_{\text{max}}(W(t)) \leq \sigma_{\text{max}}(W(0))$ when $t \rightarrow \infty$.
\end{proposition}

\begin{proof}

The optimization of the parameter matrix $W$ takes the form of $W^{t+1} = W^{t} - \alpha \nabla LI^{t}$, where $\nabla LI^t = W^{t}(x-Gx)(x-Gx)^{T}$ is a gradient of the logit invariance error (Equation \ref{eq:logit_inv}) at the time step $t$. 

Let $\epsilon = x - Gx$ and $\Sigma = \epsilon\epsilon^{T}$. With the infinitesimally small learning rate $\alpha$, we can write the gradient flow of $W$ as:

\begin{align}
    \frac{d}{dt} W = - W \Sigma
\end{align}

For a fixed $\Sigma$ we can solve the gradient flow above analytically as:

\begin{align}
    W(t) & = W(0) \exp{(-\Sigma t)}
\end{align}

Next, we consider a maximum singular value $\sigma_{\text{max}}(W) = \| W \|_2$, when the model is trained, i.e. $W(t)$ with $t \rightarrow \infty$. Applying Cauchy–Schwarz we can write:

\begin{align}
\label{eq:cauchy}
    \| W(t) \|_2 & \leq \| W(0) \|_2 \| \exp{(-\Sigma t)} \|_2
\end{align}

With a spectral decomposition $\Sigma = U \Lambda U^T$, we can write $\| \exp{(-\Sigma t)} \|_2 = \| \exp{(-\Lambda t)} \|_{2}$. Since $\Sigma$ is a rank-one matrix, it contains all zero eigenvalues except of the one, which equates to $\lambda_{\text{max}}(\Sigma) = \epsilon^T \epsilon$. Thus, eigenvalues of the matrix $\exp{(-\Lambda t)}$ are all ones except of the eigenvalue, which equates to $\lambda_{\epsilon}(t) = \exp(-t \cdot \epsilon^T \epsilon)$. Note that $\lambda_{\epsilon}(t) \leq 1$, hence $\| \exp{(-\Lambda t)} \|_{2} = 1$. Plugging into Equation \ref{eq:cauchy} gives $\| W(t) \|_2 \leq \| W(0) \|_2$ with $t \rightarrow \infty$.

\end{proof}

\textit{This reveals the non-increasing spectral norm constraint that invariance error minimization induces}. Also, initialization routines for $W$, e.g. \cite{glorot2010understanding, he2015delving}, yield small $\| W(0) \|_2 $ at the beginning of the training, further restricting the sensitivity of a network when optimizing for the low invariance.

\section{Experiments}
\label{sec:experiments}
In this section, we experimentally investigate the properties of learned group-invariance. As groups of interest we choose the $\mathbb{R}_{4}^{2}$ group of 4-fold rotations and the $\mathbb{T}_{3}^{2}$ group of 3-fold cyclic translations along the x-axis. We examine how well the learned invariance is aligned with the downstream task performance and the genuine invariance of weight-tying methods. Then, we analyze the reliability of the learned invariance under the data distribution drift. Lastly, we investigate the invariance-induced spectral decay phenomenon for various forms of the invariance error.

\subsection{Implementation details}

\begin{table}[t!]
\centering
\begin{tabular}{@{}cccccl@{}}
\toprule
\shade $\mathcal{G}$ \hspace{-2mm} & \shade Model \hspace{-2mm} & \shade Acc. (\%) & \shade $LI \Downarrow$  & \shade $DI \Downarrow$ & \shade $SI \Uparrow$ \\ \midrule
\multirow{3}{*}{$\mathbb{R}_4^2$} \hspace{-2mm} & \texttt{WT} \hspace{-2mm} & $94.6$ {\tiny$\pm0.1$} & $0.00$ {\tiny$\pm0.0$} & $0.0$ {\tiny$\pm0.0$} & $1.00$ {\tiny$\pm0.00$} \\
                   & \texttt{DA} \hspace{-2mm} & $94.0$ {\tiny$\pm0.5$} & $98.6$ {\tiny$\pm5.4$} & $0.3$ {\tiny$\pm0.1$} & $0.17$ {\tiny$\pm0.04$} \\
                   & \texttt{IR} \hspace{-2mm} &$87.9$ {\tiny$\pm0.8$} & $0.02$ {\tiny$\pm0.0$} & $0.0$ {\tiny$\pm0.0$} & $0.95$ {\tiny$\pm0.03$} \\ \midrule
\multirow{3}{*}{$\mathbb{T}_3^2$} \hspace{-2mm} & \texttt{WT} \hspace{-2mm} & $96.6$ {\tiny$\pm0.1$} & $0.0$ {\tiny$\pm0.0$} & $0.0$ {\tiny$\pm0.0$} & $1.00$ {\tiny$\pm0.0$}  \\
                   & \texttt{DA} \hspace{-2mm} & $96.2$ {\tiny$\pm0.2$} & $50.8$ {\tiny$\pm6.8$} & $0.1$ {\tiny$\pm0.0$} & $0.57$ {\tiny$\pm0.08$} \\
                   & \texttt{IR} \hspace{-2mm} & $93.1$ {\tiny$\pm0.1$} & $0.01$ {\tiny$\pm0.0$} & $0.0$ {\tiny$\pm0.0$} & $0.95$ {\tiny$\pm0.07$} \\ \bottomrule
\end{tabular}
\caption{Classification accuracy and invariance for the data augmentation [\texttt{DA}], weight-tying [\texttt{WT}], and the model trained with the logit invariance error as the regularizer [\texttt{IR}] on the Transformed-MNIST dataset. \textit{LI} - logit invariance; \textit{DI} - predictive distribution invariance; \textit{SI} - saliency invariance similarity.}
\label{tab:tabl1}
\end{table}

\paragraph{Datasets} We construct the \text{Transforming-MNIST} and \text{Transforming-FMNIST} datasets. Both dataset consist of MNIST and F-MNIST \cite{fashion2017xiao} ($28 \times 28$ black an white images of clothing categories) with $\mathbb{R}_{4}^{2}$ or $\mathbb{T}_{3}^{2}$ group transformations applied. We additionally leave out the digit $9$ from the \text{Transforming-MNIST} dataset to avoid the confusion with $6$, when studying the rotation invariance. We also leave out the last class of the \text{Transforming-FMNIST} to make number of classes equal to the \text{Transforming-MNIST}. We extend the resolution from $28 \times 28$ to $36 \times 36$ by zero-padding data samples. We use $10k / 50k / 2k$ splits for \textit{train} $/$ \textit{test} $/$ \textit{validation}. The datasets are normalized to zero mean and unit standard deviation. 

\paragraph{Models} We employ $5$-layer perceptron with \texttt{ReLU} non-linearities and the hidden dimension of $128$, resulting in total of $230k$ parameters. For the group-invariant model, we utilize group weight-tying with a pooling over a group to achieve the invariance. We only utilize group-invariant weight-tying for the first layer of a network.

\paragraph{Training details} We train all models for $300$ epochs using Adam optimizer with the batch size of $512$ and the learning rate of $0.0008$. For all models, we use data augmentation with transformations randomly sampled from a group of interest. For invariance regularization, we employ logit invariance error as a regularizer $\mathcal{I}_{f}$. We tune the weighting of the regularizer, such that the resulting saliency invariance similarity $SI_f \geq 0.95$. Final models are selected based on the best validation accuracy.

\paragraph{Saliency map function $\textbf{\textit{m}}_\textbf{\textit{f}}$} The choice of the saliency map function to compute saliency invariance score is a hyper-parameter. We employ the following procedure to generate saliency maps of a network. First, we accumulate absolute values of integrated gradients \cite{sundararajan2017axiomatic} with respect to the target class prediction. Second, we threshold the values that are less than $0.9$ of a maximum value of the absolute integrated gradient. Finally, we apply a Gaussian filter with a kernel size of $3$ and a standard deviation of $1$ to smooth the resulting saliency map.

\subsection{Learning invariance from the data}

In this experiment, we compare models that learn invariance with data augmentation [\texttt{DA}] and invariance regularization [\texttt{IR}], and models that have the invariant weight-tying built-in [\texttt{WT}]. We evaluate the models by the classification accuracy, the logit and distribution invariance errors, and also by the saliency invariance similarity. The results are reported in Table \ref{tab:tabl1}. All results are averaged over $4$ common random seeds. 

\begin{figure}[t!]
  \centering
    \includegraphics[width=0.8\linewidth]{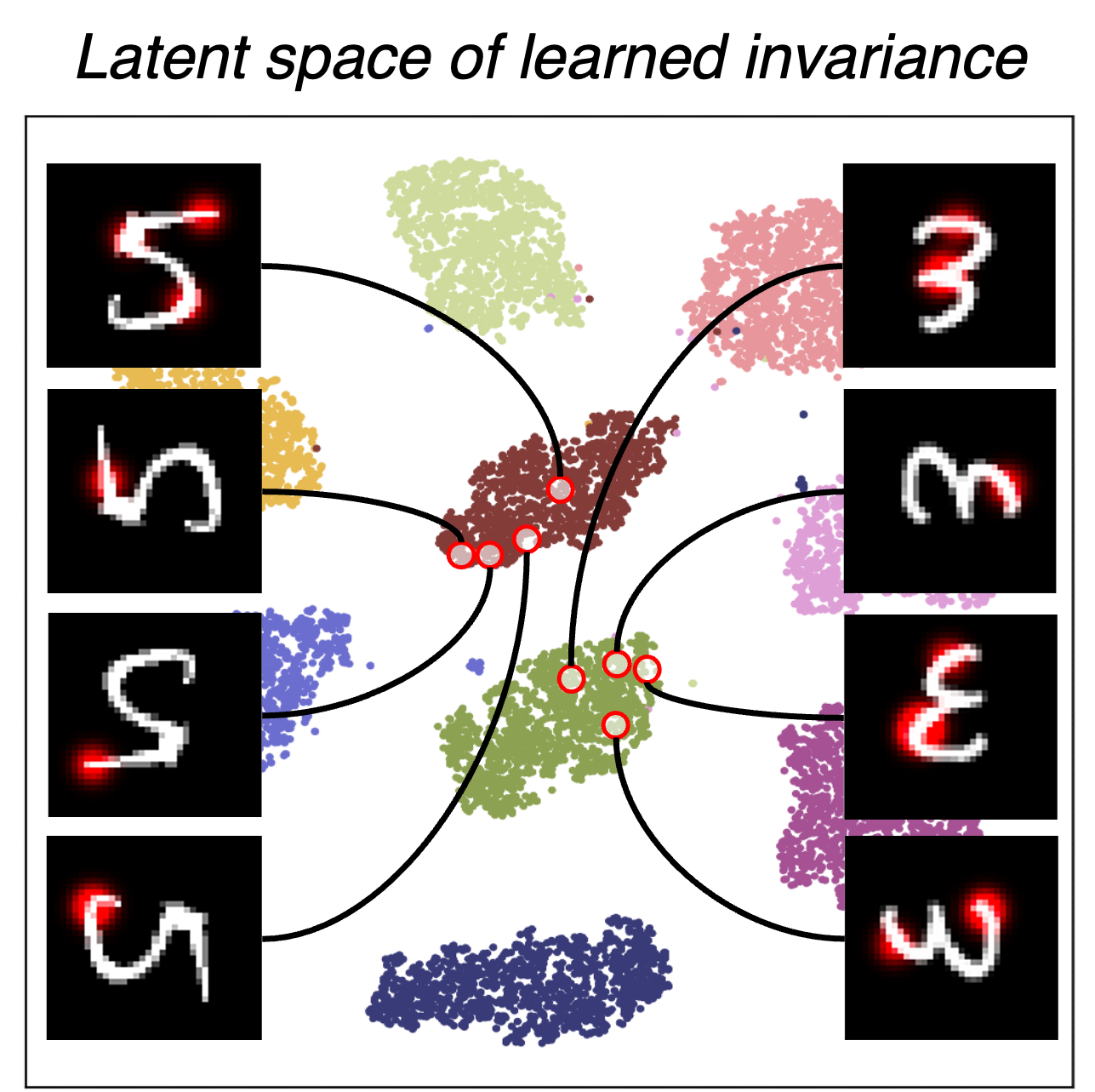}
    \caption{T-SNE of the representations in the pen-ultimate layer of the model with $\mathbb{R}_4^2$ learned invariance. Different orientations of a single sample are mapped to distant points in the latent space. The saliency of the network is highlighted by the red regions in the images.}
  \label{fig:tsne}
\end{figure}

\paragraph{\texttt{WT}} The networks with group-invariant weight tying deliver the highest classification accuracy in both $\mathbb{R}_4^2$ and $\mathbb{T}_3^2$ scenarios. We thus treat the weight-tying models as a performance upper-bound and an oracle for the genuine invariance when further analyzing models with invariance learned.

\paragraph{\texttt{DA}} We observe that the models trained with data augmentation fail to learn genuine group invariance as indicated by high logit invariance error $LI_f$ and lower saliency similarity score $SI_f$. Interestingly, these models still provide moderately low predictive distribution invariance error $DI_f$ and high classification accuracy under group transformations. This implies that \textit{neural networks can learn to solve an invariant task without learning a genuinely invariant decision making-process}.

To visualize this phenomenon, we depict the T-SNE of the latent space of the model with learned $\mathbb{R}_r^2$ invariance (Figure \ref{fig:tsne}), and we trace different orientations of a sample in the latent space. We observe that different orientations of one sample can land far away from each other in the representation space, but still within the boundaries of its class. When such configuration fully satisfies the downstream task objective, there is apparently no reason for a network to learn genuine invariance.

\begin{figure}[t!]
  \centering
    \includegraphics[width=1\linewidth]{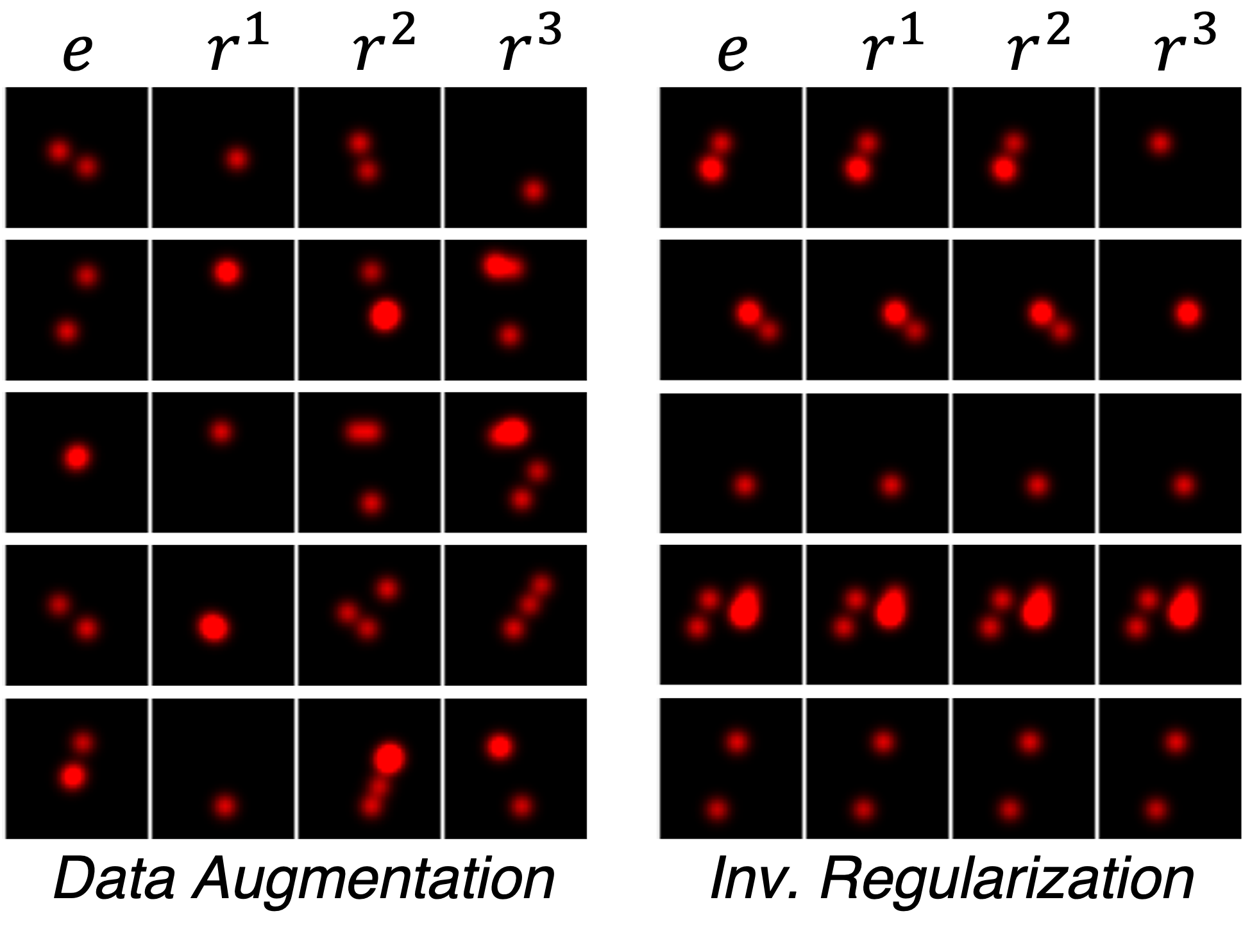}
    \caption{Saliency maps for the models with learned $\mathbb{R}_4^2$ invariance. Rows correspond to data samples and columns correspond to transformations from the group orbit applied to the sample. All saliency maps are realigned to a common orientation.}
  \label{fig:saliency}
\end{figure}

\paragraph{\texttt{IR}} The models trained with invariance regularization achieve low logit and distribution invariance errors on par with the weight-tying models. Also, high saliency invariance similarity indicates that invariance regularization guides a model towards learning genuinely invariance decision-making process. In Figure \ref{fig:saliency}, we visualize examples of saliency maps of the network with $\mathbb{R}_4^2$ invariance learned by data augmentation and invariance regularization. In contrast to the saliency maps of the model trained solely with data augmentation, saliency maps of the model with invariance regularization are well-aligned over the group orbit.

\subsection{Reliability of learned invariance}

We next investigate the reliability of the models with the learned invariance when operating conditions evolve. We simulate changing operating conditions as the data distribution drift from Transforming-MNIST to Transforming-FMNIST datasets. Practically, we linearly interpolate between those two dataset as $\mathcal{D}_{1 \rightarrow 2}(\beta) = (1-\beta) \mathcal{D}_1 + \beta \mathcal{D}_2$ to obtain the dataset with the drift degree of $\beta$. We compare the models by measuring the invariance error and the accuracy drop ratio on the drifted dataset. The accuracy drop ratio on the drifted dataset $\mathcal{D}_{1 \rightarrow 2}(\beta)$ is computed as $\text{Acc}(\mathcal{D}_{1 \rightarrow 2}(\beta)) / \text{Acc} (\mathcal{D}_{1})$, where $\text{Acc} (\mathcal{D})$ is the model's accuracy on the dataset $\mathcal{D}$. The accuracy ratio indicates how much of the original accuracy is preserved when a model is tested on the drifted dataset. The results are presented in Figure \ref{fig:drift}. 

\begin{figure}[t!]
  \centering
    \includegraphics[width=0.95\linewidth]{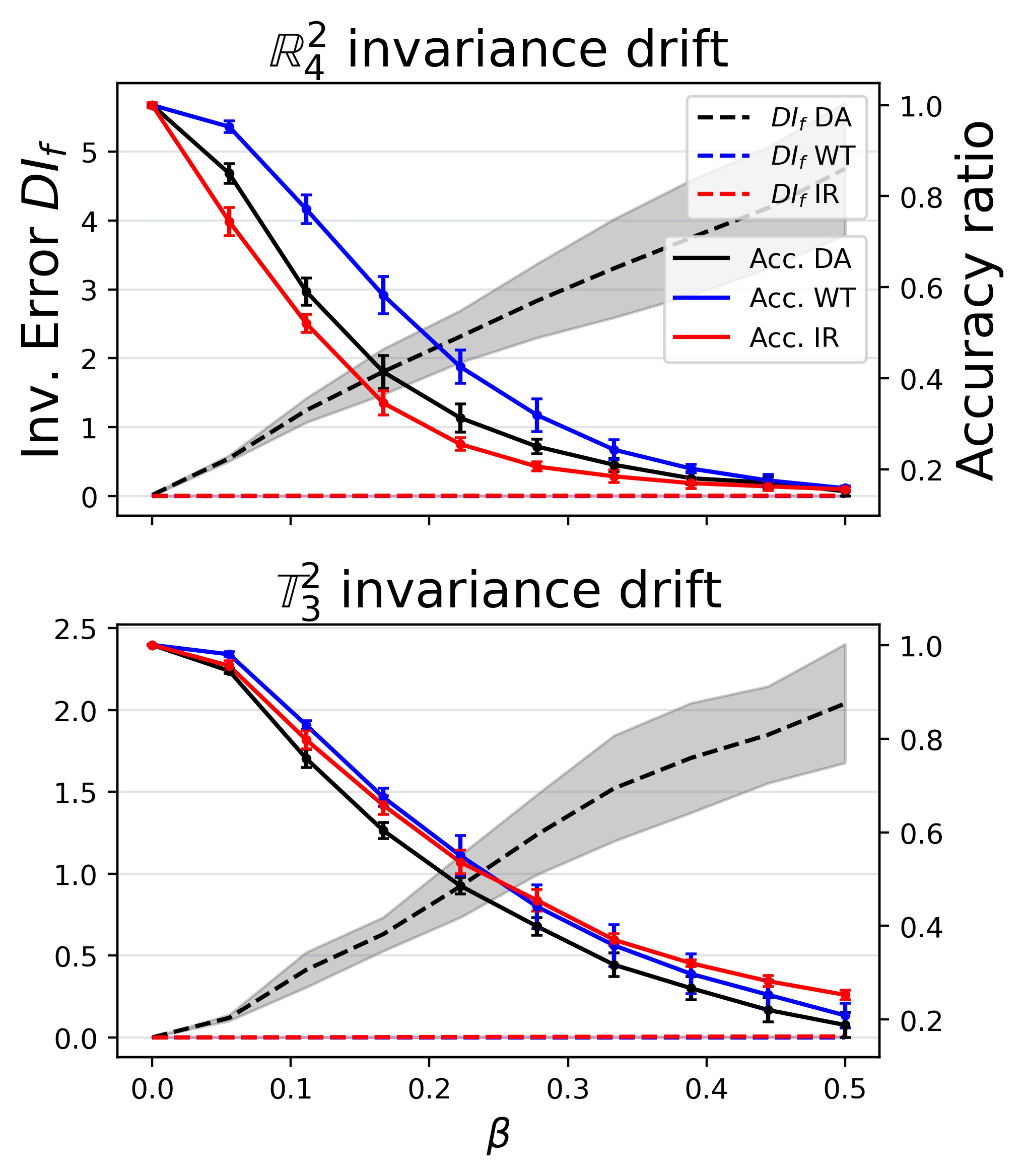}
    \caption{Predictive distribution invariance error $DI_f$ (left y-axis) and classification accuracy drop ratio (right y-axis) for the data augmentation, weight-tying and invariance regularization models over increasing degree of the data distribution drift (x-axis).}
  \label{fig:drift}
\end{figure}

\textit{We observe that the invariance learned by data augmentation deteriorates rapidly, even under a slight degree of data drift}. This turns into a major flaw if a user anticipates a certain level of invariance from the model, but then invariance instantly fails upon encountering unseen data. This also obscures the interpretability of predictions, thereby complicating the explainability of model decisions, even if accuracy is sustained. This yields invariance learned by data augmentation unreliable. \textit{Conversely, models with weight-tying and invariance regularization maintain low invariance error even under substantial distribution drift}.

Also, we observe that networks with invariant weight-tying sustain higher classification accuracy under the distribution shifts. This also holds for the models trained with invariance regularization for $\mathbb{T}^2_3$, but interestingly, not for the $\mathbb{R}^2_4$ invariance. We hypothesize this can be attributed to the \textit{accuracy-on-the-line} effect \cite{miller2021accuracy}, where models with higher in-domain accuracy tend to also deliver higher accuracy on out-of-distribution data. 

\begin{figure}[t!]
  \centering
    \includegraphics[width=0.95\linewidth]{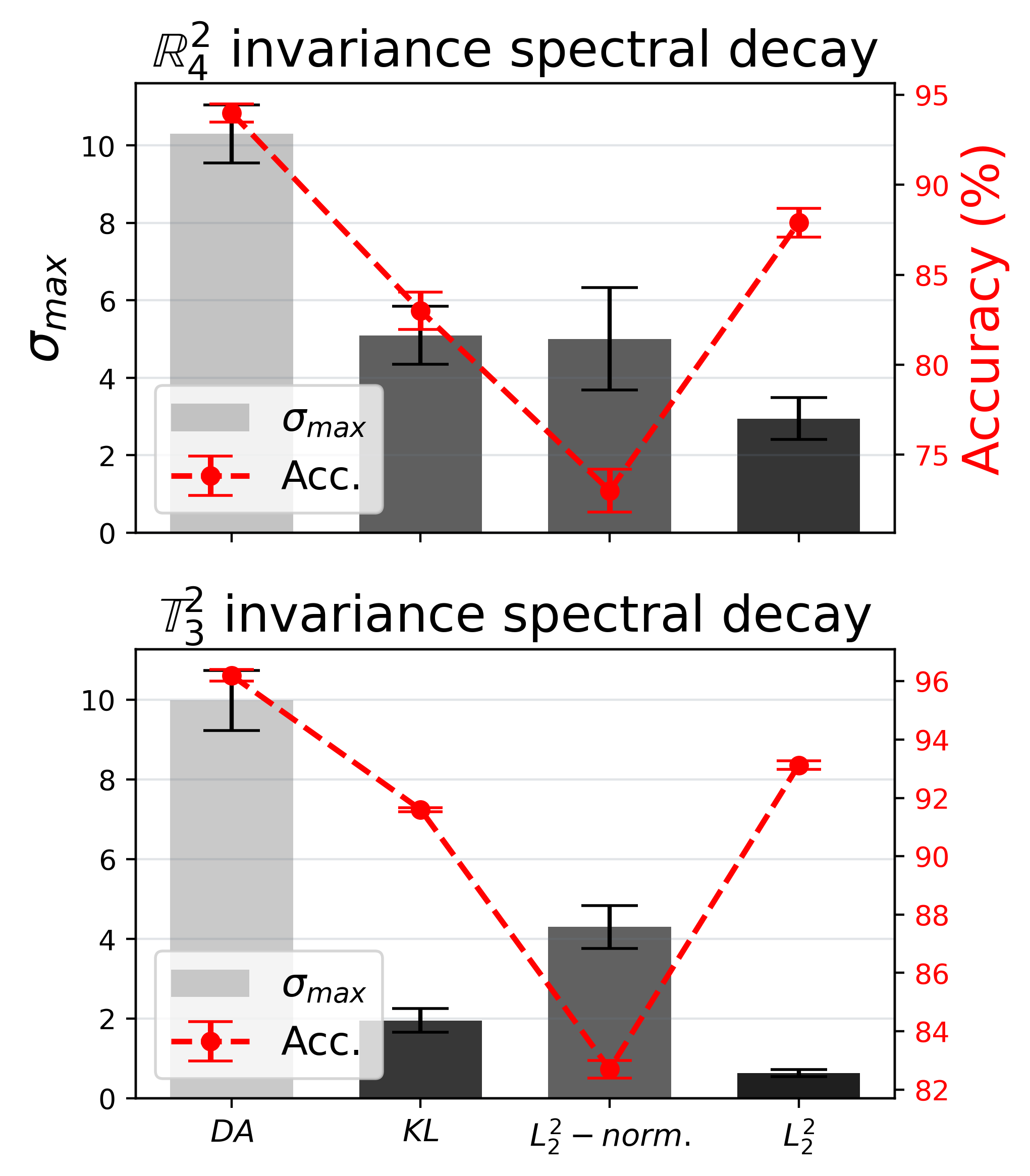}
    \caption{Sensitivity of a network to input perturbations measured by the maximum singular value of its Jacobian (left y-axis) for various forms of invariance regularization (x-axis). Models trained with invariance regularization come with overall reduced sensitivity to input perturbations.}
  \label{fig:sdecay}
\end{figure}

\subsection{Invariance-induced spectral decay}
\label{sec:sdecay}

Lastly, we take a closer look at the invariance-induced spectral decay and we verify if it holds for different forms of invariance regularization. We investigate invariance regularization with \textit{(i)} distribution invariance error with KL divergence, \textit{(ii)} logit invariance error with squared $L_2$ distance, and \textit{(iii)} logit invariance error with the squared $L_2$ distance normalized by the magnitude of the logits, i.e. $\| f(x) - f(gx) \|_2^2 / \| f(x) \|_2^2$. We tune the weighting of the regularizer for all of the models such that saliency invariance similarity $SI_f \geq 0.95$. We then evaluate the sensitivity of a network to input perturbations as a maximum singular value of its Jacobian $\sigma_{\text{max}}(J)$; and we compare it to the sensitivity of the networks trained solely with data augmentation. The results are presented in Figure \ref{fig:sdecay}.

We observe that models trained with invariance regularization come with overall reduced sensitivity to input perturbations as indicated by considerably smaller  $\sigma_{\text{max}}(J)$. This phenomenon holds for both $\mathbb{R}^2_4$ and $\mathbb{T}^2_3$ groups and various forms of invariance regularization. Note that all forms of the invariance regularization we examine also induce an accuracy drop for the model.
\section{Discussion}

\paragraph{Summary} Our study sheds light on the properties and limitations of data-driven invariance learning within neural networks. First, we proposed several measures to evaluate learned invariance: predictive distribution invariance and logit invariance errors, and saliency invariance similarity. With this, we study networks with learned group invariance and demonstrate that high performance and low invariance error do not guarantee a genuine invariant decision-making process. This leads to a notable risk, when learned invariance immediately fails beyond the training data distribution, making neural networks with learned invariance less reliable. Then, we showed that it is possible to promote genuine invariance learning by regularizing invariance during the training. Yet, such an approach leads to a spectral-decay phenomenon, when a network opts for reducing input sensitivity to all perturbations to achieve invariance to a specific group of transformations. These findings bring us a step closer to deciphering the intricate dynamics of learning inductive biases from the data.

\paragraph{Broader Impact} Our work, while primarily considering invariance to group symmetries, has potential implications for a much broader class of invariances. The increasing reliance on data-driven models, particularly in the era of large-scale machine learning, highlights the critical need to comprehend the properties of inductive biases that these models learn. Thus, understanding learned invariance, as one of the key inductive biases, becomes paramount for ensuring the fairness and interpretability of network's decisions.

\paragraph{Limitations and Future Work} While our study provides several key insights, some limitations remain. Firstly, the generalizability of our findings to other types of neural networks, other data modalities and other training regimes, e.g. self-supervised learning \cite{chen2020simple,MaskedAutoencoders2021,moskalev2022contrasting}, remains an interesting future direction to explore. Secondly, the way we design invariance regularization assumes a known group of transformations, which may not always be accessible in practice. Future work could look into methods for learning genuine invariance to unknown transformations without architectural modification. Lastly, our results also highlight a trade-off between learning the genuine invariance and the downstream task performance, opening a direction for future research into strategies for mitigating this trade-off.

\paragraph{Acknowledgements} We would like to thank Ivan Sosnovik and Volker Fischer for their helpful notes and discussions. This work has been financially supported by Bosch, the University of Amsterdam, and the allowance of Top consortia for Knowledge and Innovation (TKIs) from the Netherlands Ministry of Economic Affairs and Climate Policy. Artem Moskalev acknowledges travel support from the Video \& Image Sense (VIS) Lab of the University of Amsterdam.

\bibliographystyle{apalike}
\bibliography{main}



\end{document}